\documentclass{article}

\usepackage[preprint]{neurips_2023}

\usepackage[utf8]{inputenc} 
\usepackage[T1]{fontenc}    
\usepackage{hyperref}       
\usepackage{url}            
\usepackage{booktabs}       
\usepackage{amsfonts}       
\usepackage{nicefrac}       
\usepackage{microtype}      
\usepackage{xcolor}         
\usepackage{tikz}

\usepackage{wrapfig}
\usepackage{graphicx}

\usepackage{amsmath}
\usepackage{amsthm}
\usepackage{amssymb}

\usepackage[utf8]{inputenc}
\usepackage{hyperref}
\hypersetup{
	unicode,
	pdfauthor={Author One, Author Two, Author Three},
	pdftitle={A simple article template},
	pdfsubject={A simple article template},
	pdfkeywords={article, template, simple},
	pdfproducer={LaTeX},
	pdfcreator={pdflatex}
}

\theoremstyle{plain}
\newtheorem{theorem}{Theorem}
\newtheorem{corollary}[theorem]{Corollary}
\newtheorem{lemma}[theorem]{Lemma}

\newtheorem{assumption}[theorem]{Assumption}

\theoremstyle{definition}
\newtheorem{definition}[theorem]{Definition}

\usepackage{graphicx, color}
\graphicspath{{fig/}}

\usepackage{algorithm, algpseudocode} 
\usepackage{mathrsfs} 

\usepackage{lipsum}

\title{A Theory for Emergence of Complex Skills in Language Models}

\author{%
  Sanjeev Arora \\
  Princeton University\thanks{Work done while at Google DeepMind} \And Anirudh Goyal \\ Google DeepMind}

\begin{document}

\maketitle

\begin{abstract}

A driver of current AI research  is the fact that new skills emerge in language models when their parameter set and training corpora are scaled up.  This phenomenon is poorly understood, and a mechanistic explanation via mathematical analysis of gradient-based training seems difficult. The current paper takes a different approach, analysing emergence using the famous  (and empirical) Scaling Laws of LLMs and a simple statistical framework. Contributions include: (a) A statistical framework that relates cross-entropy loss of LLMs to competence on the basic skills that underlie language tasks. (b) Mathematical analysis showing that the Scaling Laws imply a strong form of inductive bias that allows the pre-trained model to learn very efficiently. We informally call this {\em slingshot generalization} since, naively viewed, it appears to give competence levels at skills that violate usual generalization theory. (c)  A key example of slingshot generalization, that competence at ``complex skills,'' which involve $t$-tuples of basic skills, emerges essentially at very similar scaling as competence on the elementary skills themselves.  
\end{abstract}

\vspace{-4mm}
\section{Introduction}
	\label{sec:intro}
As language models scale up, via an increase in both the number of parameters and the size of the training datasets, they exhibit remarkable new behaviors \cite{brown2020language, ganguli2022predictability, srivastava2022beyond, wei2022emergent} ---this phenomenon is often termed {\em emergence}. Some emergent properties were noticed by early model designers and have since been confirmed through experiments with substantially larger models such as GPT \cite{brown2020language, DBLP:journals/corr/abs-2303-08774}, PaLM \cite{chowdhery2022palm} and PaLM-2 \citep{anil2023palm}. Ultimate forms of emergence are {\em in-context Learning}~\cite{brown2020language} and {\em zero-shot learning}, whereby the model can understand task instructions given as part of its input and solve the task. Attempting to explain this range of phenomena solely through a mathematical analysis of gradient-based training appears challenging, and appears to call for new thinking.

The model exhibiting new ``behaviors'' upon deployment is also of obvious interest in discussions about AI safety and alignment.  A contrarian view downplaying such concerns is that  all learned behaviors were already present  somewhere in the (massive) training corpus ---the so-called ``stochastic parrots'' view~\cite{stochparrots}.

The current paper introduces new frameworks to derive mathematical understanding of such phenomena. But first we recall (well-known) challenges  in such a line of inquiry.

\begin{enumerate}
\item The phenomenon is not well-defined! Quantifying emergence of new ``skills'' requires formulating what ``language skills'' are, which is  tricky. Formalizations using { \em Probabilistic Context-Free Grammars}, {\em Boolean logic}, {\em Combinatorial Categorial Grammars},  {\em Dependency Grammars}, {\em Gricean theories},   {\em Frame Theory} and {\em Category Theory}~\cite{chomsky57,morphologyhandbook,grice75, steedman96, coeckeetal2010,tannen94} capture essential aspects of this. But it seems difficult (perhaps impossible) to integrate all of these into a single framework and connect it to statistical frameworks undelying LLMs, namely, next-word prediction using cross-entropy loss.
    
    \item Multiple skills tend to emerge roughly {\em together}. The theory has to explain how they are connected. 
    
    \item Finally, the model appears capable of flexibly combining its various capabilities during in-context learning. (See Appendix \ref{appendix:gpt_examples} for examples.)  Prior attempts to formalize ``combinations'' got highly technical, e.g. involving category theory~\cite{}.  
\end{enumerate}

\subsection{Our contributions}  \label{subsec:contributions}

We introduce new types of mathematically rigorous ---yet elementary---frameworks and analyses that stay recognizably close to current statistical frameworks, and yet offer insights into above phenomena. Since analysis of training and generalization has proved difficult, our theory  assumes  {\em LLM Scaling Laws}, which are empirically relationships describing reduction in language modeling loss with  model scale. The ``language distribution'' is conceptualized as a distribution over finite pieces of text, called {\em text-pieces}. The following are the key elements of the theory: (i) {\em A framework for conceptualizing skills.} The theory assumes there {\em exists} a set of ground-truth skills and that  text-pieces are generated by an unknown process that takes random tuples of skills and converts them into a text-piece whose comprehension requires those skills. This is best visualized as a bipartite graph  (``skill graph'') with skills on one side and text-pieces on the other. An edge $(s, t)$ means that understanding of the text-piece $t$ requires applying skill $s$. (See Figure~\ref{fig:skillgraph}.)
The framework assumes that ``understanding'' of a text-piece is testable by simple multiple-choice (``cloze'') questions (see Section ~\ref{sec:excessentropy}) inserted (by another unknown process) into the text at test time. The model does not need to predict the content of the questions ---it  only needs to predict the answer, and it is penalized for any incorrect guesses. (ii) {\em Statistical tasks associated with skills and skill-tuples:} Having cast language understanding as ability to answer cloze questions, we can define statistical tasks corresponding to  ``skills'' and ``skill $k'$-tuples'' as follows.   ``Competence'' on the skill is a fraction between $0$ and $1$ corresponding to the model's success rate when presented with all cloze questions from text-piece selected randomly among all text-pieces adjacent to the particular skill node. Competence in $k'$-tuple of skills where $k' \geq 1$ is the ability to correctly answer cloze questions in a randomly selected text-piece that is connected to all skills in the $k'$-tuple. (iii) {\em  How competencies evolve with scaling.}  Random graph theory is used to give rates at which   competencies in skills and skill-tuples improve with scaling. It is shown that competence in skill-tuples improves almost as fast as that in individual skills. (Theorem~\ref{thm:genmeasure}.)

\begin{wrapfigure}{r}{0.5\textwidth}
    \centering
    \includegraphics[scale=0.15]{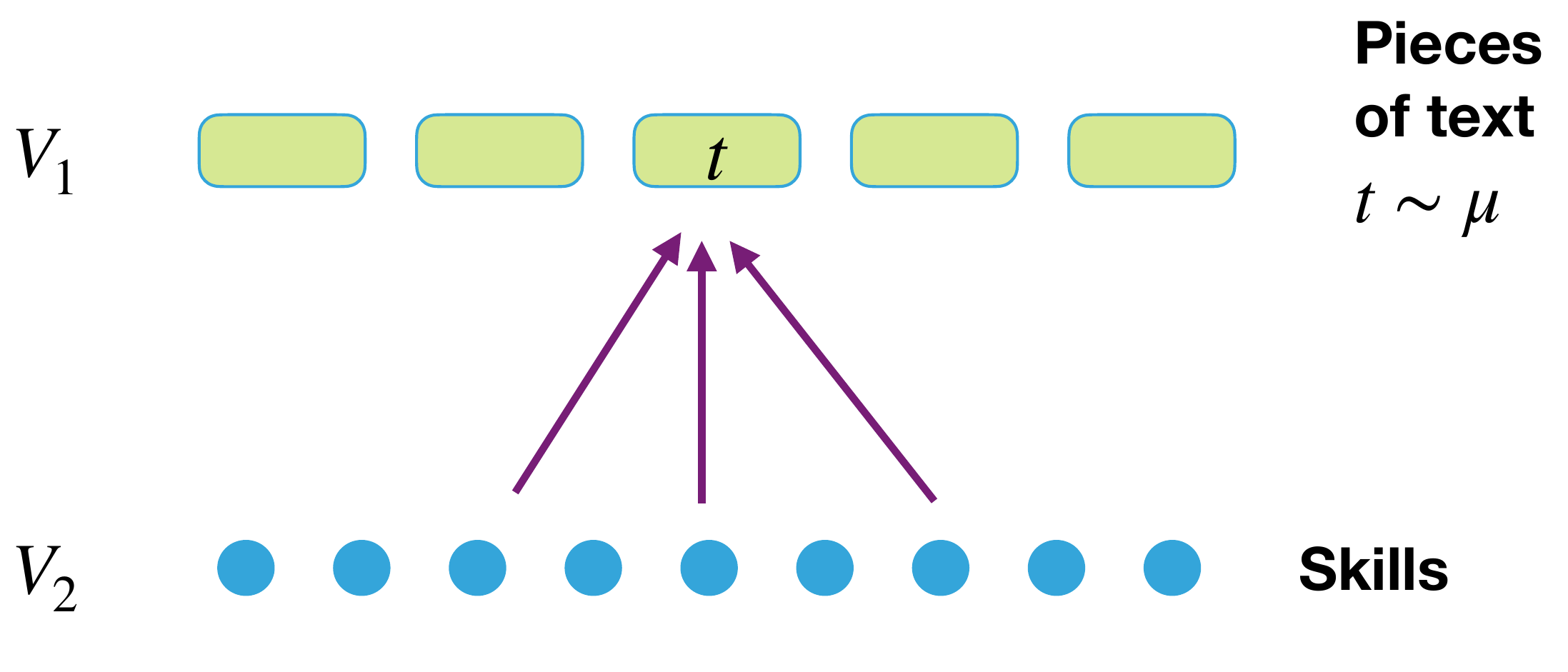}
    \caption{Skill Graph.}
    \label{fig:skillgraph}
\vspace{-3mm}
\end{wrapfigure}

\noindent{\bf {\em Poverty of the stimulus} angle:} Reasonable competence at $k'$-tuples of skills is a simple example of {\em compositional generalization}.  The phrase {\em Poverty of the Stimulus}  is often used (\cite{chomsky88})  to highlight that humans apply linguistic structures in new combinations. (See Section~\ref{sec:preliminaries}.) A similar issue also arises in our setting.  Our analysis allows the set of skills to be almost as large as the training corpus, which implies all individual skills were likely to have been seen during training. However,  the number of $k'$-tuples of skills scales as the $k'$th power of the  number of skills, which is much larger than the size of the training corpus.  Therefore if the model displays competency on even 
$10\%$ of the $k'$-tuples of skills then it must have somehow acquired competence in $k'$-tuples that were not seen during training. If we think of the $k'$-tuple of skills as a ``complex skills'' then we must conclude that the model (through some internal combinational process) has acquired many such complex skills despite not having seen that combination in training data.

\noindent{\bf Potential conceptual frameworks for next-generation AI models:} The ongoing reworking of  Language Models   into AI agents is accompanied by a shift away from the old paradigm of simply training a model to predict the next word   in a text corpus. Instead, the ``corpus'' is a carefully weighted/curated combination of  data, which could include code, math/logical reasoning, images, synthetic text, etc. Training could involve new kinds of losses.  Our conceptual framework seems applicable  to these new settings, since it is agnostic about what ``skills'' and ``text'' are, and how they compose. The analysis is also adaptable to prediction losses  other than cross-entropy.  For simplicity, we choose to describe the framework in context of vanilla language modeling.

\paragraph{Paper organization:} Section \ref{sec:preliminaries} provides a brief introduction to scaling laws, emergence, and random bipartite graphs. Section \ref{sec:excessentropy} explores the connection between the reduction in excess cross-entropy and learning.  
Sections~\ref{subsec:emergence} and~\ref{subsec:measure} give concrete results about  emergence of skills as a result of scaling.

\section{Preliminaries} 
\label{sec:preliminaries}
Deep-learning based language models follows a statistical paradigm: occurrences of linguistic units (say, words, sentences, paragraphs) are assumed to fit a statistical profile, 
and thus pieces of text from a sufficiently large and diverse corpus are assumed to be samples from a probability distribution \cite{bengio2000neural}. 
Language models are trained to solve {
\em next-word-prediction} tasks: Given, say, the past 500 words\footnote{Actual training involves breaking up words into smaller  {\em tokens}, which allows  a single model to handle all human languages,  math formulae, computer code,  etc. For simplicity, our discussion will refer to ``words.''}  in the text, predict the next word. Faced with such a task,  humans may be able to give many completions, and the modeling assumes that frequencies of these completions can be modeled via probability. 
Thus the model $M$ computes a probability $\Pr_M[w_{i+1}~|~w_{1}w_2\ldots w_{i}]$ for all possible words $w_{i+1}$'s. 
The goodness of a model is computed by its {\em cross entropy loss}, which for a sequence of words $w_1w_2\ldots w_t$ is:
\begin{equation} \label{eqn:CEloss}
\ell(M) =   - \sum_{i} \log \Pr_M[w_{i+1}~|~w_{1}w_2\ldots w_{i}] \qquad \mbox{(Cross Entropy)}
\end{equation}
Models are trained by minimizing (via gradient descent) this training loss  on a text corpus, and their goodness is computed by their {\em test loss}---evaluating the same loss expression on
a held-out text from the same corpus. Often the training corpus is so large that the model trains only once (or a few times) on each piece of text, and by the end, 
the test loss on held-out text is almost the same as the training loss.  

\paragraph{Scaling laws:}  These empirically-derived expressions describe how test cross entropy loss on held-out data scales (in experiments) with number of model parameters (N)  and size of the dataset (D) \cite{cortes1993learning, hestness2017deep, kaplan2020scaling, bahri2021explaining}. For Chinchilla models~\cite{hoffmann2022training} the law is as follows:
\begin{equation}
L(N, D) = A + \frac{B}{N^{0.34}} + \frac{C}{D^{0.28}} \qquad  A = 1.61 \quad B = 406.4 \quad C = 410.7 \label{eqn:scaling}.
\end{equation}
Here the constants $A, B, C$ in \ref{eqn:scaling}  hold only for the  specific architecture and training strategy ---even the constant  $A$  depends upon the tokenization. This description of macro behavior using two basic parameters ---reminiscent of 2nd Law of Thermodynamics--- will help us circumvent the need for mechanistic understanding of training. Our theory will only rely upon the general form of the equations, specifically, that the dependence is inverse polynomial in $N, D$. So it applies to other frameworks of training (e.g., overtrained models~\cite{muennighoff2023scaling}) where scaling laws have also been found. 

\paragraph{Emergence:}  Emergence refers to an interesting empirical phenomenon that as $D, N $ are increased together then the 
model's performance (zero shot or few-shot) on a {\em broad range} of language tasks improves in a correlated way. The improvement can appear as a quick transition when $D, N$ are plotted on a log scale (which is often the case) but it is now generally accepted that for most tasks the performance improves gradually when $D, N$ are scaled up. Thus the term {\em slow emergence} is more correct.
Furthermore,  it is known that emergence happens at different rates for different tasks, and is often quickest for tasks where the text is plausibly close to text found in training data \cite{wei2022emergent}.
Plenty of tasks are known that stump current models, and they usually tend to be very different from what one would  find in usual text corpora.  See \citep{wei2022emergent,srivastava2022beyond,schaeffermirage23} for experimental results on emergence rates of the broad range of language tasks. 
One might thus posit, with some justification from the above-mentioned studies,  that the emergence of skills arises from training on related tasks that were implicitly solved while solving next-word prediction in the training dataset. This is indeed our starting point. 

\paragraph{Poverty of Stimulus with respect to skill combinations:} Any expert who has conversed with popular chat agents quickly discovers that at the very least they seem to flexibly solve tasks that require {\em combinations} of simple skills.
However, the number of $k'$-wise combinations of elementary skills feels too large (say, with $k' =4$ and the number of elementary skills is in the tens of thousands) for all of the combinations to even be present in the training dataset.  In other words, if the model indeed acquires ability to solve tasks involving $t$-tuples of skills, we are looking at the familiar {\em Poverty of the Stimulus} issue\footnote{Chomsky coined the phrase {\em Poverty of the Stimulus}  to emphasize that babies do not have enough data to learn language, and concluded that evolution must have led to a ``universal grammar'' that humans must be be born with.
In our  framework, the language model (whose learning is initialized using Gaussian noise) also seems to defy paucity of stimulus with respect to skill combinations.} which we return to in Section~\ref{subsec:ktuples}.

\subsection{Cross-Entropy, Entropy, and Excess entropy}
\label{subsec:excessce}
\newcommand{\Ex}{\mathbb{E}}
The conceptual framework underlying  cross-entropy loss (\ref{eqn:CEloss}) is that there is a ground-truth (i.e., humans') distribution for generating the next word, which assigns probability $p_i(w~|~w_1w_2\ldots w_i)$ to the event that the $(i+1)$th word is $w$ given that the previous words were $w_1w_2\ldots w_i$. In interest of compact notation we shorten $p_i(w~|~w_1w_2\ldots w_i)$  to $p_i(w)$,
Thus the {\em entropy} of the $(i+1)$th word 
\begin{equation}
     \sum_w p_i(w) \log \frac{1}{p_i(w)} ~~~~(\text{\sc entropy})
\end{equation}
This entropy is an inherent property of language, due to existence of many possible choices human writers can make for the next word. Given sequence $w_1 w_2 \ldots w_i$ the model has a probability distribution $q(w|w_1 w_2 \ldots w_i)$ for the next word $w$.  Extending our compact notation, we use $q_i(w)$ as a shorthand for this. The  cross-entropy loss of the model on the $i+1$th word is $\log \frac{1}{q(w_{i+1})}$, which should be seen as an empirical estimate of 
\begin{equation} E_{w \sim p_i()}[\log \frac{1}{q(w)}]~~ \text{\sc (expected c-e loss)} \label{eqn:empiricalloss}
\end{equation} 
{\em KL divergence}, also sometimes called {\em excess entropy}, is  non-negative and defined as 
\begin{equation} KL(p_i|| q_i) = E_{w \sim p_i()}[\log \frac{p_i(w)}{q_i(w)}] ~~~\text{{\sc  excess entropy}}
\end{equation}
 Thus on a per-word basis we have:
 \begin{equation} \label{eqn:basicequation}
\text{{\sc expected c-e loss}}  =\text{{\sc entropy}} + \text{{\sc  excess entropy}}
\end{equation} 
Summing over the entire held out corpus, one obtains a similar estimate for the entire corpus.  One can make mild assumptions to the effect that the conditional probabilities $p_i(), q_i()$ only  depend only on (say) the previous $10^3$ words, whereas the corpus size $M$ is much bigger, e.g., $M \gg 10^7$. So the corpus consists of a random walk of sorts, where every $10^4$ words or so it switches to a different  portion of the language distribution. Under such assumptions  the above relationship, which holds in expectation at the word level, should  hold fairly precisely at the corpus level.

\paragraph{Understanding the Scaling Law in terms of excess entropy:}
In (\ref{eqn:scaling}) the $A$ term captures the entropy of language~\footnote{Here we're assuming that as the model and data set size tend to infinity in tandem, the model will perfectly learn the language distribution.}. No model, however good, can achieve lower cross-entropy loss than $A$ for large corpora. The second and third terms of (\ref{eqn:scaling}) capture {\em excess entropy}, and they decrease polynomially with $N$ and $D$. For example when $N, D$ are increased by a factor of $10$ it reduces by roughly $(10)^{0.28} \approx 2$. 

Section~\ref{sec:excessentropy} will argue that reductions in excess entropy lead to improvements in model capabilities. But note that there is no way to compute excess entropy using just the corpus. Relationship~(\ref{eqn:basicequation}) shows that estimating excess entropy requires knowing the inherent entropy of text, which requires humans in the picture.  It is not possible to look at the model's cross-entropy loss on the $i$th word  according to (\ref{eqn:CEloss}) and know---without asking humans--- how much is due to inherent cross-entropy and how much is excess.

\section{(Mis)understanding, Excess entropy, and Cloze Questions} 
\label{sec:excessentropy}

Thinking about emergence and Scaling Laws, it is possible to get confused as follows: {\em ``When we increase $D$ from $10^{11}$ to $10^{12}$ then according to~(\ref{eqn:scaling})  this  changes cross-entropy by a tiny amount. Why does it lead to big changes in macroscopic behavior?''} 
The flaw in this reasoning is that most of the loss captures merely the inherent entropy of language (the $A$ term in (\ref{eqn:scaling})). We argue now that the model's mistakes on downstream tasks (i.e., its  misunderstandings) are captured by the  {\em excess} entropy, which as noted in Section~\ref{subsec:excessce} reduces by a constant factor each time the model is scaled up by an order of magnitude\footnote{A recent empirical study~\cite{Xiatrajectories23} also concludes with the  finding that ``language modeling perplexity correlates well with few-shot in-context learning performance along the trajectory, regardless of model sizes.'' At the same time, it is known that two models with the same cross-entropy can differ somewhat in their performance on language tasks.} 

We illustrate using a classic example  from~\cite{winograd1971procedures}, which later inspired the {\em Winograd Schema Challenge(WSC)}~\cite{levesque2012winograd}:

\noindent {\tt The city councilmen refused the demonstrators a permit because they feared violence.} 

Here the pronoun {\tt they} is ambiguous--- grammar rules allow it to refer to either {\tt demonstrators} or {\tt city councilmen}. Winograd pointed out that disambiguating it (i.e., anaphora resolution) requires world knowledge that is
 unavailable in the  text itself, namely that demonstrations can get violent, and city councilmen don't like violence. 

A key idea in designing test-beds for language understanding such as WSC is  the {\bf Cloze Procedure}\footnote{Cloze questions are multiple choice, which allows testing most language skills~\cite{brown2020language}. Some skills such as understanding of irony  don't lend themselves well to cloze-based testing since one of the  multiple choices already explains the joke.  See \cite{saunshiexplore20} for earlier use of Cloze prompts in developing a theory of LLMs.}, popular also for testing language development in children~\cite{brownassessment}. To test the model's understanding of {\tt they} in this sentence, 
 we can append a {\em prompt}: {\tt Q. Who feared violence?}. This is followed by either a blank, or a choice of multiple answers:  {\tt A. city councilmen. B. demonstrators.}   For WSC examples, even though a human would be hundred percent sure of the answer, language models circa 2016 were roughly $50/50$ confused between the two options. 


In the above example, the human is $100\%$ certain of the answer, which implies their entropy here is $\log 1$, namely $0$. However if the model is split $50$-$50$ between the two options this implies it has cross-entropy  $\log 2$, all of which is {\em excess entropy}! Given the frequency of ambiguous pronouns in usual English, one concludes that a model that has not learned pronoun disambiguation will display huge excess entropy at many places in surrounding text. Thus reductions in excess entropy (which happen naturally due to scaling) will tend to squeeze out such errors. The rest of the paper tries to make this intuition mathematically precise.

Of course, text corpora  do not normally contain such artificial cloze questions. But one could imagine that the model's basic misunderstanding of the above type could, often, lead to prediction mistakes in neighboring text. Our theory in Section~\ref{sec:basictheory} will assume that cloze questions can closely capture the  model's misunderstanding.

\section{Mathematical Framework}
\label{sec:basictheory}

We give a mathematical framework for thinking about skills and how they might relate to language comprehension tasks such as pronoun disambiguation. First, it is assumed that language comprehension involves a set of skills, though the theory will not need to know a precise list.  (Scholars have  discovered and named thousands of skills. Well-trained transformers have   undoubtedly discovered many more that remain unnamed.)
Next, the theory will assume scaling laws such as (\ref{eqn:scaling}) and thus not need to reason about  training and generalization. Instead, it can reason directly about the model's behavior on the test distribution, i.e., the distribution from which the training data was drawn. We assume this test distribution is structured as a  long unordered list of text-pieces, each with an associated measure\footnote{Text-pieces should be thought of as having a size between a paragraph to a few pages, drawn from a longer corpus. To allow good prediction for the model, the text-piece could include ancillary text that preceded it the longer corpus. The model need not do predictions for the words in this ancillary text but can use it to make predictions on the text-piece.}  Traditional cross-entropy loss is averaged using this associated measure. 

\begin{definition}[Text piece] The test corpus for the model is viewed as being divided into  {\em text-pieces}, each consisting of $C_{test}$ tokens.  There is also a measure $\mu_2()$ on  these text-pieces, with $\mu_2(t)$ denoting the measure of text-piece $t$. The usual cross-entropy loss is computed by weighting text-pieces with respect to this measure. 
\end{definition}

Now we make some assumptions. We assume that the model's ``comprehension'' of a text piece is  testable via suitable cloze questions analogous to the Winograd example in Section~\ref{sec:excessentropy}. Specifically, we assume that an (unknown) process {\sc cloze} has been used to add such  cloze questions  to the text pieces at test time. These are clearly-marked multiple-choice questions in simple English that the model has to answer. Note that the training corpus did not contain such cloze questions, so this is a simple form of distribution shift at test time. The prediction loss on cloze questions does not require  predicting the location or contents of the cloze question ---it only requires  selecting the  correct answer to the  multiple-choice cloze question. 

We allow the process {\sc cloze} to tailor the questions to the model being tested. Thus the next assumption is reasonable. 

\begin{assumption}\label{assum:proportionalloss}[Cloze Sufficiency Assumption:]
{\em  The pre-trained model's average (multiclass) prediction loss on Cloze questions --- where the average is taken over the distribution of text pieces-- closely tracks (within a small multiplicative factor like 1.1)  the excess cross-entropy of the model on classical next-word prediction.} 
\end{assumption}
\noindent{\bf Note:} As discussed in Section~\ref{sec:excessentropy}, if the cloze question is assumed to be perfectly answerable by a human then any incorrect answers  by the model  can be interpreted analogously excess cross entropy. Our assumption  amounts to saying that mistakes on cloze questions closely capture the excess entropy of the model as defined in (\ref{eqn:CEloss}).  The next theorem, shows that there {\em exists} a set of cloze questions (albeit fairly artificial) where the excess cross-entropy  of the model's answer tracks the overall excess cross-entropy on next-word prediction. 

\begin{theorem}
If a model's  excess entropy at the $i$th place in text is $\epsilon$ then there is a cloze question with binary answer such that the probability that the model answers it incorrectly is at most $\sqrt{2\epsilon}$.
\end{theorem}
\begin{proof}
    The proof involves Pinsker's Inequality (wikipedia version) which relates variation distance and KL divergence. As in Section~\ref{sec:excessentropy} let $p_i()$ be the humans' probability distribution for the $i+1$th word in the text piece and $q_i()$ be the model's distribution. 
    The probability that the  human and the model give different answers is the variation distance    between the two distributions, which is the maximum (over all subsets $A$ of words) of  $\sum_{w \in A} (p_i(w) - q_i(w))$. 
    Let  $A_{i+1}$ denote the subset for which the previous expression is maximised. The cloze question consists of replacing word $w_{i+1}$ in the text with  the question: {\em Is the next word among the words listed in option (a) or in option (b)}, where option (a) lists words in $A_{i+1}$ and (b) lists words in $\overline{{A}_{i+1}}$. 
   The theorem now follows from Pinsker's inequality.
\end{proof}

\subsection{Skills: A Statistical View}

Language is assumed to have an underlying set  $S$ of {\em skills}. Every text-piece $t$ has an associated set of skills that are required for comprehending it. The theory allows this set of skills to be quite large ---it only needs to be (a fair bit) smaller than the number of text-pieces in the distribution (an enormous number).  

\begin{definition}[skill graph]  
A {\em skill graph} is a bipartite graph $(S, T, E)$ where nodes in $S$ correspond to skills, 
nodes in $T$ correspond to text-pieces,  and  $(s, t)$ is in the edge set $ E$ if ``comprehending'' text-piece $t$  (i.e., answering its associated cloze questions) requires using skill $s$. (See Figure~\ref{fig:skillgraph})
\end{definition}

It is important to realize that  we are interested in  quantifying the model's {\em competence} on a skill. For example, while the above definition assumes there the distribution of text-pieces includes those whose comprehension requires the skill ``anaphora resolution,''    a language model (or even human individuals!) will in general be unable to apply the  skill correctly in all text pieces. Thus ``competence on anaphora resolution'' is not $0/1$ ---instead it is quantified as the fraction of text-pieces associated with this skill whose cloze questions were correctly answered by the model. Quantifying the success rate of this (in other words, the model's capabilities) is the goal of the rest of the paper.

The final element of our theory is that the skill-graph has  random edges, as made precise in Definition~\ref{def:nature}. To understand why this makes sense, we recall Winograd's example: {\em The city councilmen refused the demonstrators a permit because they feared violence}.  Winograd implicitly assumes that the trickiest skill needed here is pronoun/anaphora  resolution, but of course, applying that skill in this context requires other skills: understanding of causality (i.e., interpretation of ``because'') as well as world knowledge about ``city councilmen,''  ``permit,'' ``demonstrators,'' etc.  This example highlights the fact that if we were to look at random  text-pieces that require pronoun disambiguation, we would encounter random real-world scenarios, whose comprehension requires very different set of skills. Moreover, the scenarios (and hence the relevant skills) could have different probabilities of occurring in the corpus. 

For simplicity we assume that each text-piece requires exactly $k$ skills for some $k$, and this set was drawn by iid sampling from an underlying measure on the set of skills.
(Thinking of  $k$  as a random variable is natural but will not be considered here.) 
The next definition  formalizes the above framework in form of a {\em skill cluster}.

\begin{definition}[Degree-$k$ skill cluster] \label{def:nature} This is a skill graph $(S, T, E)$ where the collection of text pieces is generated by ``nature'' by applying the following process: pick a subset of $k$ skills via iid sampling from an underlying measure $\mu_1$ on skills, and then use a procedure {\sc gen} to create a text-piece $t$ whose comprehension requires these skills, as well as a measure $\mu_2(t)$ associated\footnote{Note that the measure on text-pieces has to have the correct marginals  e.g., the $\mu_2$-measure of all text-pieces containing a skill $s$ is $\mu_1(s)$. There are many measures satisfying this weak condition, since the number of text pieces is way larger than the number of skills.} with this text piece $t$.  Then nature uses process {\sc cloze} to add cloze prompts to test comprehension on $t$. The {\em prediction loss}  on the text-piece is the cross-entropy loss on predicting the answers to the cloze questions in it. The average prediction loss over all text-pieces is computed with respect to the measure $\mu_2()$. 
We call the skill-graph thus produced  a {\em degree-$k$ skill cluster}.
\end{definition}
Now we formalize a simple model of what the full text corpus looks like. More complicated extensions of this framework (e.g., considering a hierarchy among corpora) are left for future work. 
\begin{definition}(Text corpus) The text corpus consists of many skill clusters (e.g., math, newspapers, science, coding, etc.) $(S, T_1, E_1), (S, T_2, E_2),\ldots$  which share the same underlying set of skills $S$ but have disjoint sets of text-pieces $T_1, T_2, \ldots$ that are generated as in Definition~\ref{def:nature}.
\end{definition}


Definition~\ref{def:nature} allows us to define ``competence on a skill'' in the more familiar setting of statistical learning theory, specifically by letting us associate a statistical task with it. The task involves predicting answers to cloze questions in a sub-distribution of text pieces that contain that skill.  Our emergence theory will apply to the family of tasks of the next definition.

\begin{definition}[Competence on Skills] \label{defn:statsviewskill}
In the setting of  Definition~\ref{def:nature}, for each skill cluster and each skill $s \in S$ {\em statistical task  $\tau_{s}$ corresponding to $s$ and this cluster} is defined as follows. The  learner is given a text-piece  created by sampling $s_1, \ldots, s_{k-1}$  via iid sampling $(k-1)$ times from measure $\mu_1$, and applying  {\sc gen} and {\sc cloze} to the skill-tuple $(s, s_1, \ldots, s_{k-1})$ to convert it into a text piece $t$ with an associated measure $\mu_2(t)$ (but the measure is re-scaled so that the total measure of the inputs to this task $\tau_s$ is $1$). 
The {\em error rate} of the model at the statistical tasks is the  expected prediction loss  on text-pieces drawn from the above distribution.  Since error rate is between $0$ and $1$, the {\em competence} refers to $(1 - \text{error rate})$.

For every  $k'$-tuple of skills $(s_1, s_2,\ldots, s_{k'})$ (where $k' \leq k$)  the  statistical task 
$\tau_{s_1, s_2,\ldots, s_{k'}}$ corresponding to that $k$'-tuple is similarly defined. The inputs to the task are generated by completing the $k'$-tuple to a $k$-tuple $\vec{s}$ by iid sampling  of $k- k'$ additional skills from $\mu_1$  and then using {\sc gen} and {\sc cloze} to convert it into a text-piece.

Competence on the $k'$-tuple is defined just as above.
\end{definition}
\noindent{\bf Note:} The definition involves the $k$-tuple being picked by iid sampling from $\mu_1$ which, in principle, allows a skill to be picked twice. However, the probability of picking the same skill twice scales as $O(1/|S|)$. Since the set of  skills $S$ is assumed to be large, the distribution is almost the same as sampling distinct $k$-tuples of skills. The small difference of $O(1/|S|)$ between the two methods will not affect any of the random graph theory calculations. 

To illustrate with an example, if comprehending a text-piece involves $5$ skills, then that text-piece will appear in $5$ statistical tasks corresponding to individual skills, ${5 \choose 2}$ tasks corresponding to pairs of skills, and so on.  However, our method of measuring the loss incurred on these statistical tasks implicitly assumes that if the model incorrectly answered this cloze question (i.e., it assigned significant probability to the wrong answer), then that loss was incurred in {\em all} these statistical tasks. This accounting is conservative ---it ignores the possibility that a model could have perfect on skills $1$ to $4$ but still have incorrectly answered the cloze question because of, say, shaky   understanding of skill $5$. But this conservative accounting has the significant benefit of obviating the need for a mathematical formulation of what skills are, and what  it means to combine skills ---which is unformulated, as earlier noted. In summary, Definition~\ref{defn:statsviewskill}  can be thought of as a lower bound on the model's true ``competence''  individual skills. Note this notion of competence also does not capture out-of-distribution generalization (i.e. predict well when the distribution of text pieces changes).

\section{Analysis of Emergence (uniform cluster)}
\label{subsec:emergence}

Having set up a framework for modeling skills and (via Assumption~\ref{assum:proportionalloss}) connecting them to the cross-entropy loss of the model, we have  arrived at a core mathematical issue around emergence: {\em As the  model's excess cross entropy goes down (due to scaling), this improves the model's performance on cloze tasks inserted in the test stream. 
How does this improve  competence on the skills as well as on tuples of skills --in other words, performance on the associated  cloze questions?}

This section analyzes a simple setting where the test-stream consists of a single degree-$k$ skill cluster, and the skills are uniformly distributed and so are the text-pieces---in other words, the  distributions $\mu_1$ and $\mu_2$ in Definition~\ref{def:nature} are uniform. Section~\ref{subsec:measure} will extend the analysis to the general setting.  
The calculations below only require the total number of skills to be much less than the support size of the distribution of text---in other words, the set of skills can be extremely large.

\paragraph{Key Hurdle:}  We point out the naive but incorrect way to reason about this.  Since each text piece is connected to a random $k$-tuple  of skills, say $\vec{s}$, one is tempted to reason about emergence via linearity of expectations, specifically, the following relation about prediction loss, where ``expectation'' is just average over text-pieces/skills with respect to their measure: 
\begin{equation}\label{eqn:incorrectreln}
    k \cdot E_t[\text{loss}(t)] = E_s[\text{failure rate of statistical task}~\tau_{s}].~~~ (\text{\bf Incorrect!})
\end{equation}
To see that this is incorrect,  let $Y$ be the subset of such text pieces where the model makes mistakes on cloze questions. This $Y$ depends upon the skill graph, and the unknown processes {\sc gen} and {\sc cloze} of Definition~\ref{def:nature}, which assign measure to text pieces in an unknown way that may introduce arbitrary correlations. Since the model ``saw'' part of the test stream (namely, the portion corresponding to training data) it has picked some  information about the skill cluster. Thus at the end of training, locations of errors in the test stream  --i.e., the set $Y$--- depend upon the skill-cluster, and since we lack understanding of $Y$ the analysis has to treat it as arbitrary. In other words, our analysis is allowed to assume an upper bound on the test loss, but the text-pieces on which this loss occurs form an arbitrary subset that depends upon the graph structure. In particular, (\ref{eqn:incorrectreln}) cannot be inferred. This is the key mathematical hurdle and our proof will surmount it using  random graph theory.

Let's say the model {\em makes a  mistake} on a text-piece if the total prediction loss on all the  cloze-questions of that text-piece is at least $1/2$ (which is the kind of error incurred if the incorrect answer is chosen with noticeable probability  on even a  single cloze question). If the average cross-entropy loss for the text-pieces is $\delta$ we conclude $Y$ consists of at most $2\delta$ fraction of text pieces.  The following result guarantees that statistical tasks corresponding to most skills do not assign significant probability to text pieces in $Y$ --in other words, the model has good performance on statistical tasks connected with these skills. The theorem follows from (and is a simple rephrasing of) Lemma~\ref{lem:mixing} in the appendix. 
\begin{theorem}[Basic]\label{corr:emerge1}
 Let $\alpha, \beta, \theta >0, \beta >1, \alpha \beta <1, \theta <1$ satisfy
\begin{equation} \label{eqn:mix2a}
    H(\theta) + k\theta \left( H(\beta \alpha)  - \beta \alpha  \log \frac{1}{\alpha} - (1- \beta \alpha)  \log (\frac{1}{1-\alpha})\right)<0
\end{equation} and  the distribution on skills and text pieces be uniform in the skill-cluster. Then irrespective of the details of {\sc gen} and {\sc cloze} processes, the following property holds for every subset $Y$ of text pieces  that contains at least $\theta$ fraction of text pieces:  at least $1-\alpha$ fraction of skills have at most $\beta \theta  k N_1/N_2$ edges to $Y$ (in other words, at  most $\beta$ times the number of edges a skill would be {\em expected} to have to text-pieces in $Y$).
 \end{theorem}
 Note that as the model is scaled up, $\theta$ will go down and the set $Y$ containing erroneous answers on cloze questions will shrink.  Our analysis kicks in only once  $\theta$ drops below $1$. In terms of the emergence phenomenon, this corresponds to first signs of improvement on downstream tasks once the model's loss drops below some threshold.
 
  Since edges between a skill node $s$ and set $Y$ correspond to errors in the statistical task $\tau_s$,  Theorem~\ref{corr:emerge1} is giving an upper bound on the prediction error in statistical tasks corresponding to $(1-\alpha)$ fraction of skills. 

 \begin{definition}[performance curve]
The contour plot (i.e., the boundary) of the region of $\alpha, \beta$ combinations satisfying Theorem~\ref{corr:emerge1} is called a {\em performance curve} and denoted $C_{(k,\theta)}$. A performance curve $C$ is {\em better} than another curve $C'$ if for every $\alpha, \beta$ on $C$ there is a corresponding point $(\alpha, \beta')$ on $C'$ for $\beta' >\beta$.
 \end{definition}

 Figure~\ref{fig:sub1} gives {\em performance curves}, i.e., the contour plot of the set of
$\alpha, \beta$ combinations satisfying Theorem~\ref{corr:emerge1} for a given $\theta, k$.     The horizontal axis plots $(1-\alpha)$ and the vertical axis plots $\beta \theta$, so  point $(0.8, 0.16)$ on a curve    means at least $0.8$ fraction of skills have at most $0.16$ fraction of their edges in the ``error set'' $Y$ (hence $0.84$ fraction of their edges are outside the error set). The emergence curves shift down noticeably (i.e., imply emergence of more skills) as we increase $k$. The next lemma shows this trend always holds; follows from the fact that 
$H(\theta)/\theta$ is a decreasing function in the interval $(0, 1)$. 

\begin{lemma}[Monotonicity] If $\theta' < \theta$ then the performance curve for $\theta', k$ lies below that for $\theta, k$. 

If $k' > k$ then  the performance curve of $\theta, k'$ lies  below that for $k, \theta$.
\end{lemma}

\begin{figure}[h]
  \centering
  \begin{minipage}{0.45\linewidth}
    \centering
    \includegraphics[height=1.5in]{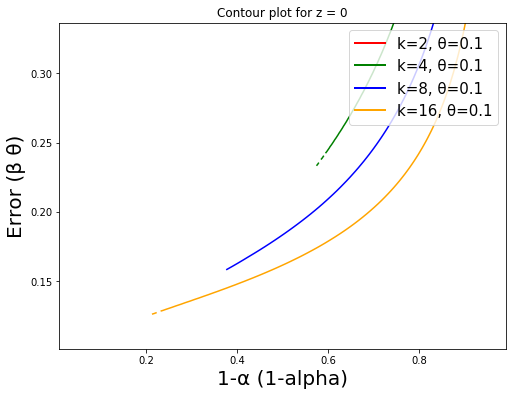}
    \label{fig:sub1}
  \end{minipage}
  \hfill
  \begin{minipage}{0.45\linewidth}
    \centering
    \includegraphics[height=1.5in]{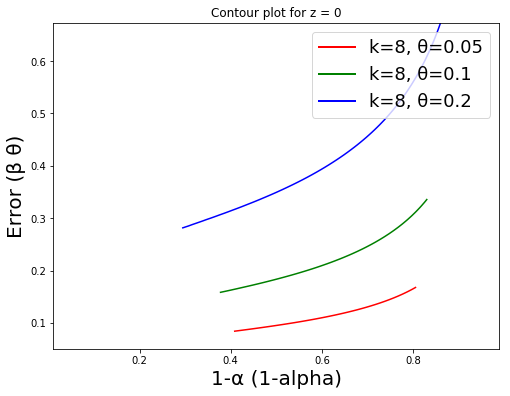}
    \label{fig:sub2}
  \end{minipage}
 \caption{Performance Curves: Left plot  has $\theta =0.1$ and varies $k=2, 4, 8, 16$.  Higher values of $k$ greatly improve performance (for $k=2$ valid $\alpha, \beta$ did not exist). 
The right plot has $k=8$ and  $\theta=0.05, 0.1, 0.2$.   Section~\ref{subsec:ktuples} clarifies that it also describes the model's performance curve for $t$-tuples of skills for for $\theta =0.05$ and $t=1, 2, 4$ respectively (e.g., blue curve for $4$-tuples).}
  \label{fig:mixing_lemma}
  \vspace{-1mm}
\end{figure}

\subsection{The tensorization argument}
While the above method yields performance curves,  better curves can be derived via a tensorization argument. 
Consider the following {\em $k'$-wise recombination} operation on the test stream. First randomly partition the test stream into subsets of size $k'$, and then concatenate  the $k'$ text pieces within each subset to create a larger piece of text that we refer to as a ``$k'$-piece,'' and whose measure  is the sum of the measures of the component test-pieces. All cloze questions for the old test-pieces are retained and no new cloze questions are inserted. Clearly, if the error of the model per average text-piece was $\delta$, then the  error per average $b$-piece is $k'\delta$. 
However, each $k'$-piece is now using  a random $k'k$-tuple of skills. Importantly, this set of $k'k$ skills consists of iid draws from the skill distribution. In other words, Theorem~\ref{corr:emerge1} now becomes the following.

\begin{corollary}[tensorization] \label{corr:emerge2} In the same setting as Theorem~\ref{corr:emerge1}, for integer $k' \in [2, 1/\theta]$ the conclusion of that theorem holds also for $\alpha, \beta$ pairs satisfying
\begin{equation} \label{eqn: tensor1}
 H(k'\theta) + kk'\theta \left( H(\beta \alpha)  - \beta \alpha  \log \frac{1}{\alpha} - (1- \beta \alpha)  \log (\frac{1}{1-\alpha})\right) < 0
\end{equation}   
Furthermore, if $H(k'\theta) < k'H(\theta) $ the emergence curve from this expression dominates that derived from Theorem~\ref{corr:emerge1}.
\end{corollary}

\subsubsection{Emergence for $k'$-tuples of skills}
\label{subsec:ktuples}
Now we estimate the model's emergence curve for statistical tasks corresponding to $k'$-tuples for $k'\leq k$.
The basic idea is to consider $k'$-tuples of skills as `composite-skills,' and then re-do the calculation.

\noindent{\bf 2nd estimate (better):} Consider the following {\em $k'$-wise recombination} operation on the test stream. First randomly partition the test stream into subsets of size $k'$, and then concatenate  the $k'$ text pieces within each subset to create a larger piece of text that we refer to as a ``$k'$-piece.''  All cloze questions for the old test-pieces are retained and no new cloze questions are inserted. Clearly, if the error of the model per average text-piece was $\delta$, then the  error per average $b$-piece is $k'\delta$. 
However, each $k'$-piece is now using  a random $k'k$-tuple of skills,   which we can alternatively view as $k$  random $k'$-tuples. Thus viewing $k'$-tuples of skills as `composite skills' we can use this as the skill set in  the setting of Theorem~\ref{corr:emerge1}, which gives us an easy corollary quantifying the   performance on tasks corresponding to $k'$-tuples of skills. 

\begin{lemma}[Emergence for $k'$-tuples of skills] \label{corr:emerge2}
     Consider the skill-graph $(S', T', E)$ where $S'$ consists of all $k'$-tuples of skills, $T'$ consists of $k'$-pieces, and $E$ consists of $(s', t')$ where $s'$ is a $k'$-tuple of skills and $t'$ is a $k'$-piece where this tuple of skills is used. Let $Y$ consist of $\theta$ fraction of $k'$-pieces. Then for any $\alpha, \beta >0, \beta >1, \alpha \beta <1$ satisfying (\ref{eqn:mix1}) there are at least $1-\alpha$ fraction of $k'$-tuples of skills that have at most $\alpha \beta \theta \theta N_1$ $\beta \theta$ fraction of their edges connected to $Y$.
\end{lemma}

The next corollary presents a somewhat surprising general principle that's also hinted at in caption of Figure~\ref{fig:mixing_lemma}. Assume (for simplicity) a Chinchilla-like scaling law that $10$x up-scaling leads to  factor $2$ reduction in excess entropy. If a model is considered to have reasonable performance on individual skills at current scaling, then after further up-scaling of $10x$ one would see similar reasonable performance on skill-pairs, and scaling up by yet another $10$x after that will yield similar reasonable performance on $4$-tuples of skills, etc. 
Note that these are {\em provable lower bounds} on performance gains---actual gains could  be higher. 
Figure~\ref{fig:mixing_lemma} illustrates the phenomenon.

\begin{corollary} \label{corr:emerge3} When the model $M_1$ with loss $\delta$ is scaled up (e.g., as per equation~(\ref{eqn:scaling}))  so that the new model $M_2$ has loss  $\delta/k'$,  then the performance curve inferred by our method for $k'$-tuples of skills using $M_2$ is identical to the curve inferred for individual skills on model $M_1$.
\end{corollary}  
\begin{proof}
    As noted above, a loss of $\delta$ still allows the model to make significant mistakes on $2\delta$ fraction of test pieces, which we denote by $\theta$. Thus Theorem~\ref{corr:emerge1} describes the performance curve for skills. 
    Making the loss drop to $\delta/k'$ but creating $k'$-pieces makes the fraction of errors $\theta =2\delta$ again. (Note that ``error'' now means an erroneous answer  on {\em any} cloze question in the entire $k'$-piece ---again, this is a conservative definition of error.) Applying Lemma~\ref{corr:emerge2} we get the same emergence curve as Theorem~\ref{corr:emerge1}.
\end{proof}


\vspace{-2mm}
\section{Emergence analysis with general  measure on text and skills}
\label{subsec:measure}
\vspace{-2mm}

Now we turn to analysis of the general setting of Definition~\ref{def:nature} where text piece $t$ has measure $\mu_2(t)$ and  skill $s$ has measure $\mu_1(s)$.  In this setup, our lemma statements (e.g., Lemma~\ref{lem:mixing} as well as the ones in Sections~\ref{subsec:emergence} and \ref{subsec:ktuples}) hold -----the claim is the same but with cardinalities replaced by measure!

\begin{theorem}[Emergence of skills and $k$'-tuples of skills] \label{thm:genmeasure} Let $Y$ be any subset of text pieces consisting of text pieces with total measure $\theta$, and every text-piece has measure substantially less than $\theta$.  Let $\alpha, \beta >0, \beta >1, \alpha \beta <1$ satisfy
\begin{equation} \label{eqn:mix2a}
    H(\theta) + k\theta (H(\beta \alpha)  - \beta \alpha  \log \frac{1}{\alpha} - (1- \beta \alpha)  \log (\frac{1}{1-\alpha}))<0
\end{equation}
 Then the measure of skills that have at most $\beta \theta$ fraction of their edges connected to $Y$ is at least $1-\alpha$.

 For $k'$-tuples of skills the statement of Lemma~\ref{corr:emerge2}  holds with the same modification of cardinality to ``measure.''
 \end{theorem}
\begin{proof}
  The measure $\mu_1$ on skills is trivial to reason about  by just replacing each skill $s$ by a  number of copies that is proportional to $\mu_1(s)$. This converts the measure to a uniform measure ---specifically, $k$ iid draws  from this uniform measure are equivalent to $k$ iid  draws from the  $\mu_1$.
  
For the measure $\mu_2(\cdot)$ on texts, the above trick doesn't work. Recall that a text-piece is connected in the skill graph to a random $k$-tuple of skills. If we try to replace $\mu_2()$ with a uniform measure by replacing the text piece with identical copies, then these copies must still all connect to the {\em same}  subset of $k$ skills ---meaning these connections are correlated and not random. We need a more subtle argument. The key part in the proof of Lemma~\ref{lem:mixing} is where we
choose  random subset of text-pieces, $Y$  whose size is $\theta |T|$ and subset $Z$ of skills of size $\alpha |S|$, and then upper bound by (\label{eqn:mix1}) the expectation of the event that the latter has more than $\alpha \beta \theta k$ fraction of its edges going to
$Y$.  In presence of measure $\mu_2()$ let's pick $Y$ as follows: Independently pick text-pieces, choosing $t$ with probability $\theta \mu_2(t)$. (Note: $|Y|$ is  tightly concentrated around $\theta |T|$.) We still pick $Z$ randomly as before. Then we apply Jensen's Inquality on the same calculation to end up with the same upper bound as before. See Lemma~\ref{lem:mixing+measure} in the Appendix.
\end{proof}

\vspace{-2mm}
\subsection{Extending theory to multiple clusters}
\label{subsec:multipleclusters}
\vspace{-2mm}
Above we assumed a single skill cluster in the language. Real-life text  might contain multiple skill clusters. For example,  standard corpora must contain a large skill cluster involving pieces of  ``everyday'' text pieces 
and a set of basic language skills and  world knowledge needed to comprehend them. Smaller clusters may correspond to specialized topics, e.g., finance, science, mathematical reasoning, etc.  We assume each piece of text appears in only one cluster but skills may appear in different clusters. When each text-piece appears in a single cluster, the  analysis of Section~\ref{sec:slingshot}) continues to apply.  The overall loss is the weighted sum of measure of text in the individual clusters. Thus overall reduction in loss will drive emergence within individual clusters. But lacking any mechanistic insight, our theory cannot predict the rate at which loss decrease (and hence emergence) happens within clusters. This pertains to the point made earlier in the paper about lack of detailed study of scaling laws for different kinds of corpora, as well as for training on mixes of corpora.



We leave a more fine-grained analysis, including possibly allowing hierarchical structure in clusters, for future work. As usual, simpler settings probably give the main insight. 

\section{Takeaways about skill emergence}
\label{sec:skilltakeaways}

It may be useful to note the following takeaways about skill emergence as per  our theory.

\noindent{\bf 1. How scaling improves competence on $k'$-tuples of skills:}
Theorem~\ref{thm:genmeasure} and Corollary~\ref{corr:emerge3} implies that the effect of reducing $\theta$ by a factor $2$ (which as per scaling laws corresponds to roughly one order of scaling up in model parameters) 
has the effect of raising competence on $2k'$-tuples to at least the same level as what it was on $k'$-tuples before scaling.

\noindent{\bf 2. Effect of  using ``high quality'' text:} Theorem~\ref{corr:emerge1} shows that for a fixed prediction loss $\theta$, using higher $k$ implies better emergence of skills. Since $k$ is the number of skills being used in a single text-piece, it intuitively measures how {\em complex} the text is ---e.g., a college text would be expected to have higher $k$ than a primary school text.   If the scaling law is same for both types of text (i.e., how $\theta$ reduces from scaling) our theorem predicts that more complex text  will be more effective at inducing skills.  This prediction generally matches experts' intuition, although we are not aware of a study  of scaling laws that tries to separate out texts of different quality. 

\noindent {\bf 3. More frequent skills tend to reach competence level quicker than less frequent skills:} This effect is  hidden in the proof of Theorem~\ref{thm:genmeasure}. Specifically, the proof reduces the case of skills appearing with different frequencies in the corpus to the uniform case by replacing a skill node with a set of nodes whose cardinality scales in proportion to the skill frequency. But note that by definition, the competence on all copies of the same skill must be the same. Thus essentially the calculation says that $k'$-tuples  that include more frequent skills will tend to emerge faster.

\noindent{\bf 4. Learning despite Paucity of stimulus.}  We discuss how the improvement of
competence on $k'$-tuple of skills (as discussed in item 1. above) leads to a paucity of stimulus situation. 
Suppose we trained a language model with $D$ tokens. After scaling by $k$ orders of magnitude (i.e., increasing dataset size to $c^{k'}D$ tokens, where in the Chinchilla framework $c$ is around $10$) the performance on $k'$-tuples of skills is as good as what the performance was on  individual skills before the scaling. Note that the number of $k'$ tuples of skills is around $|S|^{k'}$ where $S$ is the set of skills.  This quickly leads to paucity of stimulus for some fixed $k'$, specifically, if $Dc^{k'} \ll |S|^{k'}$. We give an example just for illustration. Suppose $c =10$ and $|S|=10^4$ and the model's proficiency on individual skills was considered good when it was trained with $D = 10^{10}$ tokens (roughly the dataset size for GPT-2 style models).   Then a larger model trained with $10$ trillion tokens ($10^{13}$) --  closer to the size of corpora used in training today's models-- would display proficiency in most $8$-tuples of skills, despite never not having seen most of those combinations in training (which we can be sure of because $10^{10} \times 10^8 \ll (10^4)^8$).

\vspace{-2mm}
\section{Conclusions}
\vspace{-2mm}

We  have proposed a theoretical framework for  understanding   emergence of  skills when language models are scaled up. A key insight (see Figure~\ref{fig:mixing_lemma}) is that reduction in excess cross entropy loss drives skill acquisition, together with the assumption that normal language ---down to short paragraph level---already utilizes multiple skills, mixed up randomly. Need for mechanistic insight is sidestepped using    Scaling Law, which quantifies a powerful inductive bias in pre-trained models.
One concrete example of this inductive bias is that in our framework proficiency in combinations of skills  arises just as naturally as proficiency in the individual skills themselves, and need not require seeing examples of all (or even most) of these combinations in the training set. This has relevance to the ongoing debate about the extent of ``understanding'' that current models have, and their ability to address novel settings.

We hope the simplicity of our framework will also encourage further experimental and theoretical study, including extensions to more general language skills such as generation and dialog; and modeling inductive bias at a finer level than the Scaling Laws. (For example, what are the scaling laws for interesting parts of language such as math or coding?) It is also possible that our theory underestimates the rate of emergence,  due to unknown mechanisms ---e.g., having to do with workings of transformers--that are left out in our theoretical framework.

The simple and statistical nature of our theory should be seen as a plus --- it helps identify which emergence phenomena should not be considered surprising, most notably emergence of competence on skills as well as on their combinations. But it shares limitations with other statistical frameworks. Competence is guaranteed only on text-pieces drawn from the data distribution, and governed by usual $\epsilon$-$\delta$) guarantees --- many skills as well as combinations of skills may not get learnt, and the ones that do get learnt may incorrectly applied on some fraction of the data distribution.   Nevertheless we hope this inspires more thorough experimental study (our simple experiments give a starting point) of whether or not current language models have capabilities that go beyond simple statistical explanations.   Empirical properties or phenomena that are not derivable in our framework (or its natural extensions) may be of interest for AI alignment as well as better design and understanding of  language models.

\noindent{\bf Acknowledgements:} We are very grateful to Jonah Brown-Cohen, Timothy Lillicrap and Melvin Joshnson for many discussions that motivated us to improve the theory and its expositions. We thank Boaz Barak, Rong Ge, Yuxi Liu, and Nikunj Saunshi for their feedback on the manuscript. 


\bibliographystyle{plainnat}
\bibliography{refs}

\newpage 
\appendix
\section{Appendix}

\subsection{Graphs of Emergence Phenomena}
For context, we include some graphs of emergence of capabilities from \cite{wei2022emergent}.

\subsection{Example of current chatbots' ability to combine skills}
\label{appendix:gpt_examples}
We find that current chatbots, including those in the public domain, can take a list of language skills and produce text illustrating those skills. When the list includes harder (or less common) skills, this ability  declines. We include an illustrative example but a more thorough evaluation is left to future work.

\begin{tikzpicture}
\node[rectangle, draw=black, thick, inner sep=10pt]{
\begin{minipage}{\textwidth}

\colorbox{cyan}{\textbf{Human question:}}  Please give me a a couple lines of text that illustrate all of the following language understanding skills: Anaphora resolution, simple logical reasoning, simple understanding of physics, and understanding of sentiment.

\colorbox{yellow}{\textbf{Model response:}}  Sure! Here's a couple of lines that illustrate the mentioned language understanding skills:

"As the sun set, casting a warm glow over the horizon, John watched in awe. He couldn't help but feel a sense of tranquility wash over him, as if all his worries had melted away with the fading daylight."

Explanation:
\begin{itemize}
    \item Anaphora resolution: The pronoun "he" in the second sentence refers back to the antecedent "John" in the first sentence, resolving the anaphoric reference.
    \item Simple logical reasoning: The logical reasoning here is based on the cause and effect relationship between the setting sun and John's feeling of tranquility. The implied logical connection suggests that the calming effect of the sunset is responsible for John's serene state of mind.
    \item Simple understanding of physics: The mention of the sun setting and casting a warm glow over the horizon reflects a basic understanding of physics, as it relates to the natural phenomena of the Earth's rotation and the resulting visual effect.
    \item Understanding of sentiment: The phrase "a sense of tranquility" and the description of John's worries melting away indicate a positive sentiment associated with the experience. The mention of John feeling awe and a sense of tranquility further emphasizes a positive emotional state.
\end{itemize}

\colorbox{cyan}{\textbf{Human question:}}  Okay. Can you change the above example to introduce the element of "hip-hop".

\textbf{\colorbox{yellow}{\textbf{Model response:}} } Certainly! Here's an updated example that incorporates the element of "hip-hop":

"As the sun set, casting a warm glow over the city skyline, the rhythmic beats of hip-hop music filled the air. John stood mesmerized, his body instinctively moving to the infectious grooves. He couldn't help but feel a surge of energy and joy, as if the music had unlocked a vibrant world within him."

Explanation:

\begin{itemize}
\item Anaphora resolution: The pronoun "he" in the second sentence still refers back to the antecedent "John" in the first sentence, resolving the anaphoric reference.
\item Simple logical reasoning: The logical connection here lies in the cause and effect relationship between the hip-hop music and John's feeling of energy and joy. The music's rhythmic beats and infectious grooves are responsible for his instinctive movement and the surge of positive emotions he experiences.
\item Simple understanding of physics: Although the example doesn't directly mention physics, the phrase "As the sun set, casting a warm glow over the city skyline" still implies the natural phenomena of the sun's descent and the resulting illumination of the environment.
\item Understanding of sentiment: The introduction of hip-hop music adds a vibrant and energetic sentiment to the example. John's feeling of energy, joy, and being mesmerized by the music reflects a positive sentiment associated with the hip-hop genre.
\end{itemize}

\end{minipage}
};
\end{tikzpicture}

\section{Technical Theorems about Random Bipartite Graphs}

\label{sec:randomgraphs}
The theory will need some facts about random bipartite graph $(V_1, V_2, E)$  with $N_i$ denoting $|V_i|$, and $N_1 \gg N_2$. 
When we say it has degree $k$, we mean that in $E$ every vertex in $N_1$ is connected to $k$ vertices in $N_2$, where those $k$ vertices are chosen i.i.d.\ with replacement. Recall that $V_1$ corresponded to text-pieces and $V_2$ to skills in the main body of the paper.

The next lemma uses the famous Probabilistic Method~\cite{alonspencer}. In this method, one is trying to show that in a certain probability space, there are no {\em bad} outcomes.  This is done by letting $W$ be an integer random variable denoting the number of bad outcomes that happened, and showing that the  $E[W] \approx 0$. Then it follows that  $W =0$ with probability at least $1 - E[W]$. Concretely, in the next Lemma  $W$ will be the number of ``bad'' set pairs  $(Y,Z)$ of a certain size that violate the lemma. 

\begin{lemma}\label{lem:mixing}
    For every positive integer $k$ and $\theta \in [0,1]$  there are $\alpha, \beta >0$ such that $\alpha \beta \leq 1$ and the following holds with probability almost $1$. For every $Y \subseteq V_1$ of size $\theta N_1$, there are at least   $(1-\alpha)$ fraction of vertices in $V_2$ each of which has  at most $\beta \theta D$ edges going to $Y$, where $D =kN_1/N_2$ is the expected degree of a node in $V_2$. The parameter values for which this occurs are specified by the condition
\begin{equation} \label{eqn:mix2b}
    H(\theta) + k\theta (H(\beta \alpha)  - \beta \alpha  \log \frac{1}{\alpha} - (1- \beta \alpha)  \log (\frac{1}{1-\alpha}))<0
\end{equation}\end{lemma}
\begin{proof}
For $Y \subseteq V_1, |Y| =\theta N_1$ and $Z \subseteq V_2, |Z|\leq  \alpha N_2$ we say that $(Y, Z)$ are  {\em bad} if $Z$ has at least $\alpha \beta \theta k N_1$ edges to $Y$.  Let $W$ denote the  number of such $Z$'s. The expectation is upper bounded by 
\begin{equation} N_2 N_1 {N_2 \choose \alpha N_2} \times {N_1 \choose \theta N_1} \times {k\theta N_1 \choose \beta \alpha k \theta  N_1} \times \alpha^{\beta \alpha \theta k N_1} \times (1-\alpha)^{(1-\beta \alpha )\theta kN_1} \label{eqn:mix1} \end{equation}
For (\ref{eqn:mix1}) to be $\ll 1$ it suffices for its logarithm to be negative. 
By Stirling's approximation ${N \choose t N} \leq 2^{(H(t) + \epsilon_N)N}$ where $H(t)= - t\log t - (1-t)\log (1-t) $ is the binary entropy function and $\epsilon_N$ goes to zero rapidly as $N \rightarrow \infty$. 
Applying this to (\ref{eqn:mix1}) and  taking logarithms, and assuming $N_2 \ll N_1$, we  arrive at the condition (\ref{eqn:mix2}) for large $N_1$.
\end{proof}
\noindent{\bf Note:} Such arguments allow a fair bit of slop. The expectation was  exponentially small, and then we took its logarithm and then divided out by a large number, $N_1$ to reach (\ref{eqn:mix2b}). Thus additional  polynomial factors  in the expectation ---such as $N_1N_2$ above--- have no effect on asymptotics. 

\subsection{Analysis of ``General Measure'' case in Section~\ref{subsec:measure}}

We give more details of the proof of Theorem~\ref{thm:genmeasure} in Section~\ref{subsec:measure}.
Again, we phrase it via a general lemma about bipartite graph 
$(V_1, V_2, E)$ where each vertex in $V_1$ has edges to $k$ random vertices in $V_2$. We use the shorthand $N_i =|V_i|$. As noted in proof of Theorem~\ref{thm:genmeasure} it suffices to consider the case when $V_2$ has uniform measure and there is a general measure $\mu()$  on vertices of $V_1$, namely
 $\mu(v_1)$ is nonnegative and $\sum_{v_1 \in V_1} \mu(v_1) =1$. The measure of an edge $(v_1, v_2)$ is defined as $\mu(v_1)$. We assume all  $\mu(v_1)$ are sufficiently small.

The proof will use a discretization of the measure. We conceptually divide $V_1$ (and hence also the set of edges) into {\em classes}, where the $i$th class ${C}_i$ consists of 
$v_1$ such that $\mu(v_1) \in [(1+\epsilon)^{-i -1}, (1+\epsilon)^{-i})$ for $\epsilon$ an arbitrarily small constant. We assume all  $\mu(v_1)$ are sufficiently small (meaning some large-ish $i_0$, class $i$ is empty for $i < i_0$) and the number of nonempty levels is much smaller than $N_1$. Thus each class has reasonable size ---say,  much larger than $N_2$, the number of skills---which allows  the asymptotic arguments appearing below to hold within each class.  
 The above assumptions all seem reasonable for the probability measure associated with text pieces, which should be fairly well spread out.

\begin{lemma}[Main]\label{lem:mixing+measure}
     For every positive integer $k$ and $\theta \in [0,1]$  and $\alpha, \beta >0$ satisfying $\alpha \beta \leq 1$ and 
    \begin{equation} \label{eqn:mix2}
    H(\theta) + k\theta (H(\beta \alpha)  - \beta \alpha  \log \frac{1}{\alpha} - (1- \beta \alpha)  \log (\frac{1}{1-\alpha}))<0
    \end{equation}
    the following holds with probability almost $1$ in the random bipartite graph $(V_1, V_2, E)$ of degree $k$:

 For every $Y \subseteq V_1$ of total measure $\theta$, there is a set of least   $(1-\alpha)$ fraction of vertices in $V_2$ such that  for each $v_2$ in this set, 
\begin{equation}
    \sum_{v_1: (v_1,v_2) \in E, v_2 \in Y}\mu(v_1) \leq \beta \theta. \label{eqn:failm1}
\end{equation} \end{lemma}
\begin{proof}

  Consider $Y  \subseteq V_1$ that has measure $\theta$, and $Z \subseteq V_2$ has size $\alpha N$.  We say   $(Y, Z)$  is {\em bad} if every $v_2 \in Z$ fails condition (\ref{eqn:failm1}). (Consequently, the measure of edges between $Z$ and $Y$ is at least $\alpha \beta \theta$.) We will argue that in the random graph, the expected number of bad $(Y, Z)$ is $\ll 1$. In other words, for any fixed measure $\mu$ with high probability the graph contains no bad $(Y, Z)$.
  (As explained in the note following Lemma~\ref{lem:mixing}, we can ignore the contribution to the expectation of 
  $Z$'s that have size $< \alpha N$.) 

For any fixed $Y \subseteq V_1$ we denote  $Y \cap  C_i$ as $Y_i$ and let $y_i =|Y_i|$. 
If $\mu(Y) =\theta$ then the $y_i$'s satisfy the following
\begin{equation}
    \sum_i y_i (1+\epsilon)^{-i} \in [\theta, (1+\epsilon)\theta]. \qquad \label{eqn:failm2}
\end{equation}

For a fixed $(Y, Z)$ let $\beta_i$ be such that the number of edges between $Y_i, Z$ is $\alpha \beta_i y_ik$. Then the probability (over the choice of the random graph) that $(Y, Z)$ is bad is at most:
\begin{equation}
    \label{eqn:mixm3}
    \prod_i {ky_i \choose \alpha \beta_i ky_i}(\alpha)^{-\alpha \beta_i k y_i}(1-\alpha)^{ky_i(1 -\alpha\beta_i)} 
\end{equation}
Since $\sum_i y_i =|Y|$ and ${ky_i \choose \alpha \beta_i ky_i} \approx 2^{H(\alpha \beta_i)ky_i}$, the left hand side is an expression of type 
$$\prod_i 2^{k y_i(H( \alpha \beta_i) -  \alpha\beta_i \log \frac{\alpha}{1-\alpha} + \alpha \log \frac{1}{1-\alpha})}.$$  Using first order optimality wrt $y_i$'s, this is maxiized when all $\beta_i$'s are equal. So for deriving an upper bound it suffices to let all $\beta_i =\beta$,
which simplifies (\ref{eqn:mixm3}) to
\begin{equation}
    \label{eqn:mixm4}
    \prod_i {ky_i \choose \alpha \beta ky_i}(\alpha)^{-\alpha \beta y_i}(1-\alpha)^{ky_i(1 -\alpha\beta)} \approx 2^{H(\alpha \beta)k|Y|} (\alpha)^{-\alpha \beta |Y|}(1-\alpha)^{k|Y|(1 -\alpha\beta)} .
\end{equation}
Now we finish the proof using reasoning similar to that in Lemma~\ref{lem:mixing}. The number of choices for $y_1, y_2, \ldots, $ is $\prod_i |C_i|$, which is at most $N_1^{P}$ where $P$ is the number of classes. 

For a fixed sequence of $y_i$'s the number of sets  $Y$ consistent with those intersections is 
$$ \prod_i {|C_i| \choose y_i} \approx \prod 2^{H(y_i)} .$$
Since $y_i$'s satisfy \eqref{eqn:failm2} and $H()$ is a concave function, this number is maximised when $y_i/|C_i| \in [\theta, \theta (1+\epsilon)]$, and hence the number of possible $Y$'s is upper bounded by 
\begin{equation}  \prod_i 2^{H(\theta |C_i|)} = 2^{H(\theta N_1)}. \label{eqn:numclasses} \end{equation}

By the union bound, the probability that there exists a $Y$ such that  $(Y, Z)$ is bad is at most $2^{H(\theta) N_1}$ times (\ref{eqn:mixm4}) times  (\ref{eqn:numclasses}). Since $|Y| \approx \theta N_1$ this completes the proof of Lemma~\ref{lem:mixing+measure}.
\end{proof}

\end{document}